\documentclass[11pt]{article}
\usepackage{acl}

\usepackage{times}
\usepackage{latexsym}
\usepackage[T1]{fontenc}
\usepackage[utf8]{inputenc}
\usepackage{microtype}

\usepackage{amsmath,amssymb}
\usepackage{booktabs}
\usepackage{multirow}
\usepackage{array}
\usepackage{enumitem}
\usepackage{graphicx}
\usepackage{placeins}
\usepackage{colortbl}
\usepackage{xcolor}
\usepackage{amsthm}

\newtheorem{theorem}{Theorem}
\newtheorem{proposition}{Proposition}

\definecolor{bestcolor}{HTML}{E3F2FD} 
\definecolor{secondcolor}{HTML}{E8F5E9} 

\title{StepOPSD: Step-Aware Online Preference Self-Distillation for Agent Reinforcement Learning}

\author{
Yanfei Zhang$^{1}$, Xu Lin$^{2}$, Chenglin Wu$^{3}$ \\
$^{1}$Independent Researcher \\
$^{2}$Tencent \\
$^{3}$DeepWisdom
}

\newcommand{\method}{StepOPSD}
\newcommand{\grpo}{GRPO}

\newcommand{\searchr}{Search-R1}

\newcommand{\best}[1]{\cellcolor{bestcolor}\textbf{\underline{#1}}}
\newcommand{\second}[1]{\cellcolor{secondcolor}\textbf{#1}}

\begin{document}
\maketitle

\begin{abstract}
Reinforcement learning for multi-turn agents suffers from a credit-assignment mismatch: rewards are sparse and trajectory-level, while success often hinges on a few local decisions. Existing online policy distillation (OPD) provides denser token-level supervision, but typically treats heterogeneous agent trajectories as monolithic strings rather than causal interaction units. We present \method{}, a post-rollout preference self-distillation framework that takes the agent step as the unit of credit redistribution. \method{} decomposes trajectories into action-centered step segments, rescoring them under hindsight-enriched teacher contexts and converting token-level log-probability gaps into sign-preserving advantage shaping with a normalized per-step credit budget before the \grpo{} update. Across ALFWorld and Search-QA with \texttt{Qwen3-1.7B} and \texttt{Qwen2.5-3B-Instruct}, \method{} attains best or second-best results on subsets most sensitive to local causal errors, including first-place performance on ALFWorld \texttt{Heat} ($79.1\%$), \texttt{PickTwo} ($95.0\%$), Search-QA \texttt{TriviaQA} ($61.6\%$), and tied-best performance on \texttt{HotpotQA} ($40.4\%$). The results further reveal a consistent two-knob law: smaller $\alpha_{clip}$ acts as a broadly stabilizing local trust region, whereas the optimal global mixing strength $\lambda_{mix}$ remains task-dependent. These findings suggest that step-aware distillation is most useful when trajectory-level rewards are weakly aligned with the local action that determines downstream success.
\end{abstract}

\begin{figure*}[t]
  \centering
  \IfFileExists{stepopsd_review.png}{
    \includegraphics[width=\textwidth]{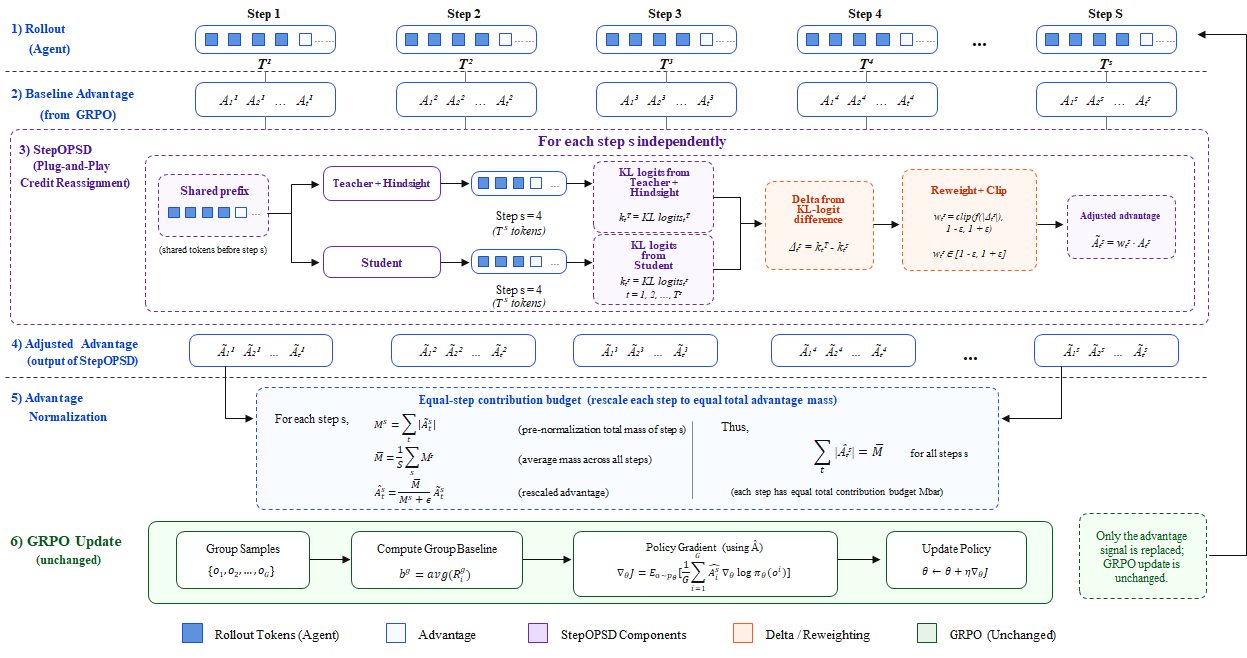}
  }{
    \fbox{
      \parbox{0.95\textwidth}{
        \centering
        Figure placeholder. Save the overview image as
        \texttt{stepopsd_review.png} under the current ACL template
        directory to render it here.
      }
    }
  }
  \caption{Overview of \method{}. \method{} decouples online environment interaction from offline credit shaping. While the base \grpo{} rollout dynamics remain untouched, post-rollout trajectories are structurally parsed into causal action boundaries. By rescoring these isolated spans under hindsight-enriched teacher contexts, \method{} translates token-level preference gaps into a direct modulation of the RL advantage—injecting dense, step-aware supervision without distorting the primary trajectory-level reward signal.}
  \label{fig:stepopsd-overview}
\end{figure*}

\section{Introduction}
Agentic post-training for Large Language Models (LLMs) is increasingly shaped by two complementary paradigms \citep{yao2023react,schick2023toolformer,jin2025search}: Reinforcement Learning (RL), which optimizes task outcomes from environment feedback, and teacher-guided distillation, which supplies denser token-level supervision. In long-horizon agents \citep{wu-etal-2026-spark,yang2026opid,wu2026seed}, however, naively combining these signals creates a severe credit-assignment mismatch. A single sparse reward must supervise dozens of decisions, even though success often hinges on only one or two of them.

This mismatch has two structural sources. \textit{First, failure is highly localized.} A malformed query, an invalid action, or a premature answer can ruin a rollout spanning thousands of tokens, yet standard RL broadcasts the same trajectory-level signal across the full response. \textit{Second, agent trajectories are heterogeneous.} Unlike static reasoning traces \citep{wu-etal-2026-beyond-examples}, they interleave immutable environment observations with model-generated action spans. Treating the full trajectory as a uniform string wastes supervision on non-controllable tokens and blurs the causal boundaries that make hindsight feedback useful.

These observations suggest that the key problem is not merely how to inject a stronger teacher, but \emph{where} credit should be redistributed once hindsight information is available. For multi-turn agents, the natural unit is not the whole response but the action-centered step. A useful distillation signal should therefore respect the step structure of interaction and concentrate optimization mass on the local decisions that actually determine downstream outcomes.

Driven by this view, we present \method{} (Step-Aware Online Preference Distillation). The term \emph{preference} is intentional: unlike policy distillation, our method does not ask the student to imitate the teacher distribution. Instead, it uses OPSD-style teacher--student rescoring to induce bounded local preferences over realized step tokens, while the RL advantage still determines the update direction. In this sense, the teacher does not replace the policy objective; it only reshapes where credit is assigned within the sampled trajectory.

Operationally, \method{} decomposes completed trajectories into causal step segments, synthesizes hindsight-enriched teacher contexts for those spans, and converts teacher--student log-probability gaps into step-aware advantage shaping before the \grpo{} update. Unlike RLSD \citep{yang2026rlsd}, which re-weights all tokens using self-divergence, or SDAR \citep{lu2026sdar}, which applies soft token-level gating over the full sequence, \method{} localizes supervision to the action-centered spans that actually drive agent behavior. The intervention is non-intrusive only in the narrow sense that it never alters the online rollout dynamics; it does, however, actively perform post-rollout, sign-preserving advantage shaping through a specialized reward manager and a lightweight trainer subclass.

To summarize, our primary contributions are three-fold:
\begin{itemize}[leftmargin=1.5em]
    \item \textbf{Step-Aware Online Preference Distillation:} We introduce \method{}, which shifts distillation from monolithic responses to action-centered step segments and repurposes teacher rescoring as post-rollout credit redistribution rather than policy imitation.
    \item \textbf{Modular, Non-Intrusive Architecture:} A drop-in implementation atop \searchr{}/\grpo{} that integrates diverse hindsight sources and span-selection strategies without changing online rollout dynamics.
    \item \textbf{Empirical Validation:} \method{} attains multiple best and second-best results on the subsets most sensitive to local decision errors, including first-place scores on ALFWorld \texttt{Heat} and \texttt{PickTwo}, and on Search-QA \texttt{TriviaQA} and \texttt{HotpotQA}. Beyond average gains, the results reveal a two-knob law: tighter local clipping through smaller $\alpha_{clip}$ is broadly stabilizing, whereas the optimal global mixing strength $\lambda_{mix}$ remains task-dependent, with weaker shaping favoring embodied control and stronger shaping favoring retrieval-centric QA.
\end{itemize}

\section{Method}
\subsection{The Credit Assignment Bottleneck}
The standard paradigm for agentic fine-tuning relies on Group Relative Policy Optimization (\grpo{}) \citep{schulman2017ppo,deepseekr1}, where a policy $\pi_{\theta}$ generates a multi-step trajectory $\tau = (s_1, a_1, o_1, \dots, s_T, a_T)$, and a single terminal reward $R(\tau)$ is broadcast across all tokens. This is fundamentally misaligned with the nature of long-horizon tasks. A rollout spanning thousands of tokens often fails due to a single localized error—a malformed query, an invalid action, or a premature conclusion. Standard RL blindly smears the penalty across the entire sequence, forcing the agent to guess which step was fatal. Instead of learning a dense value model (which is notoriously unstable and prone to hallucination in agentic domains), \method{} addresses this bottleneck through post-rollout hindsight distillation. Rather than introducing a standalone objective, \method{} acts as a surgical intervention within the existing pipeline:
\begin{equation}
\begin{split}
&\text{rollout} \rightarrow \text{reward} \rightarrow \text{step extraction} \rightarrow \\
&\text{teacher--student rescoring} \rightarrow \\
&\text{advantage shaping} \rightarrow \text{policy update}.
\end{split}
\end{equation}
By decoupling the online interaction from offline credit shaping, we inject step-aware supervision directly into the \grpo{} advantage, preserving the stability of the base RL algorithm while surgically correcting localized errors.

\subsection{Isolating Causal Action Spans}
Agent trajectories are highly heterogeneous, interleaving immutable environment observations ($o_t$) with policy-generated actions ($a_t$). Applying distillation to the entire trajectory wastes modeling capacity on observation tokens that the policy cannot control. To resolve this, we structurally parse the completed trajectory into atomic step segments aligned with task-specific tags (e.g., \texttt{<action>...</action>}). In our implementation, we strictly employ \texttt{action\_only} extraction for embodied tasks (ALFWorld), where rigid commands are the sole causal drivers. For knowledge-intensive tasks (Search-QA), we utilize \texttt{clean\_step\_no\_observation}, which includes the agent's internal reasoning but masks out the retrieved external knowledge.

\subsection{Hindsight-Privileged Rescoring}
To correct a failed step, the agent must understand what it should have done. We achieve this by contrasting the agent's standard generation probability against a hindsight-privileged teacher. For each extracted step $k$, we define a student context $c_k^S$ (the causal prefix) and a teacher context $c_k^T = c_k^S \oplus h_k$, where $h_k$ is the injected hindsight information. 

Concretely, \textbf{each step is scored as a self-contained generation unit}: its prefix is the history up to that step's boundary, and the hindsight $h_k$ is injected only at that boundary. The step is thus re-evaluated under the agent's native conditional distribution $\pi(\cdot \mid c_k^S)$, rather than as part of a monolithic trajectory. 

While hindsight theoretically spans multiple forms—from a weak binary success flag to an expensive oracle demonstration—we exclusively adopt \textit{peer-trajectory hindsight}. When a \grpo{} group contains both successes and failures, we condition the teacher context of a failed trajectory on the first successful peer from that exact same group. When a group contains no successful peer, $h_k$ degrades to a final-correctness-only cue (a binary success/failure flag), so the teacher reduces to the standard outcome-conditioned setting. This critical design choice provides a dense, informative cue without incurring the distribution shift or annotation cost of external oracle demonstrations.

\subsection{Quantifying Step-Level Credit}
We quantify the localized ``regret'' of an action by measuring how much the hindsight information alters the likelihood of the generated step. For each token $z_{k,j}$ within step $k$, we compute the log-probability gap between the hindsight-aware teacher $\pi_T$ and the student $\pi_S$:
\begin{equation}
\begin{split}
\Delta_{k,j} =& \log \pi_T(z_{k,j} \mid c_k^T, z_{k,<j}) \\
&- \log \pi_S(z_{k,j} \mid c_k^S, z_{k,<j}).
\end{split}
\end{equation}
Because $\pi_S$ and $\pi_T$ differ only in their prefix, $\Delta_{k,j}$ is measured under the agent's native conditional distribution $\pi_S(z_{k,j} \mid c_k^S, z_{k,<j})$ --- a per-step, per-token quantity, rather than a single global score over the whole rollout. 

This $\Delta$ isolates critical decision nodes: tokens where the teacher strongly diverges from the student. To prevent moving-target instability during continuous RL, we instantiate the teacher $\pi_T$ as a \texttt{stale\_ref\_policy} (the policy parameters from a fixed number of steps ago, e.g., 10 steps), ensuring the log-probability gap reflects a stable reference distribution.

\subsection{Credit-Aware Advantage Shaping}
The log-probability gap $\Delta$ must be translated into an optimization signal that modulates the \grpo{} advantage without destroying its core mathematical properties. Let $A_{\ell}$ be the base token-level advantage. We first construct a raw multiplicative weight and then project it onto a symmetric local trust region controlled by $\alpha_{clip}$:
\begin{equation}
\begin{split}
w_{\ell}^{\mathrm{raw}} &= 2 \cdot \sigma(\mathrm{sign}(A_{\ell}) \cdot \Delta_{\ell}), \\
w_{\ell} &= \Pi_{[1-\alpha_{clip},\, 1+\alpha_{clip}]}\!\left(w_{\ell}^{\mathrm{raw}}\right),
\end{split}
\end{equation}
The final shaped advantage mixes the base RL signal with the hindsight weight via a decay parameter $\lambda_{mix}$:
\begin{equation}
\tilde{A}_{\ell} = (1-\lambda_{mix}) A_{\ell} + \lambda_{mix} (w_{\ell} A_{\ell}).
\end{equation}
This formulation strictly preserves the sign of the original RL advantage. It only redistributes the \textit{magnitude} of the optimization mass—amplifying updates on tokens where the teacher agrees, and dampening updates where the teacher diverges.

\subsection{Causal Equivalence via Step Normalization}
Finally, because steps naturally vary in token length, a purely token-wise advantage shaping would allow a long, verbose reasoning step to mathematically dominate a short, concise action step, simply by accumulating more $\Delta$ mass. We counter this by introducing an \texttt{equal\_step\_mean\_abs} constraint. The shaping weights within each extracted step are rescaled such that the mean absolute modification budget is uniform across all steps in the trajectory. The philosophical rationale is explicit: in the absence of dense sub-step rewards, the most principled prior is that each reasoning step serves as an equally critical causal link to the final outcome. The $\Delta$ signal is only used to distribute this fixed budget \textit{internally} among the tokens of a given step, preventing verbosity from hijacking the credit assignment.

\section{Experiments}

\begin{table*}[t]
\centering
\small
\resizebox{\textwidth}{!}{
\begin{tabular}{lccccccc|cccccccc}
\toprule
Method & Pick & Look & Clean & Heat & Cool & Pick2 & Avg & NQ & Triv & Pop & Hotp & 2Wk & MuS & Bam & Avg \\
\midrule
\multicolumn{16}{l}{\textit{Qwen3-1.7B}} \\
Vanilla & 25.0 & 22.2 & 3.1 & 0.0 & 21.4 & 4.2 & 12.5 & 29.4 & 46.9 & 37.0 & 23.5 & 19.6 & 6.4 & 10.5 & 24.8 \\
Skill-Prompt* & 10.3 & \second{50.0} & 16.1 & 0.0 & 5.0 & 9.4 & 14.1 & 29.4 & 46.5 & 36.2 & 22.9 & 20.8 & 4.3 & 10.1 & 24.3 \\
OPSD & 26.3 & 33.3 & 9.1 & 0.0 & 4.5 & 5.3 & 14.1 & 4.2 & 8.3 & 4.6 & 6.6 & 15.3 & 0.7 & 1.2 & 5.8 \\
GRPO & \second{71.1} & 41.7 & 36.4 & 40.0 & 31.8 & 31.6 & 42.1 & 40.0 & \second{58.9} & 43.5 & 35.4 & 30.3 & 12.0 & 65.7 & 40.8 \\
Skill-GRPO & 27.6 & \best{54.5} & 22.7 & 27.3 & 0.0 & 19.2 & 21.1 & 39.2 & 58.6 & 43.9 & 35.2 & 28.2 & 11.5 & \second{66.1} & 40.4 \\
Skill-GRPO* & 31.4 & 42.9 & 51.9 & 8.3 & 11.5 & 7.1 & 28.1 & 38.0 & 58.4 & 43.9 & 36.3 & 29.0 & 12.5 & \best{66.9} & 40.7 \\
GRPO+OPSD & 38.2 & \second{50.0} & 30.8 & 28.6 & 30.0 & 21.1 & 32.0 & \best{40.7} & \second{58.9} & 45.0 & \second{37.0} & \second{34.6} & \best{13.3} & 65.7 & \best{42.2} \\
Skill-SD & 52.9 & 37.5 & \second{69.2} & \second{42.9} & \best{60.0} & \second{36.8} & 52.3 & 39.1 & 57.5 & \second{45.4} & 34.8 & 34.1 & 10.7 & 64.1 & 40.8 \\
RLSD & 50.0 & 37.5 & 61.5 & 19.0 & 50.0 & 21.1 & 42.2 & 38.6 & 57.3 & 43.0 & 34.5 & 34.1 & 11.5 & 65.3 & 40.6 \\
SDAR & \best{73.5} & 25.0 & \best{76.9} & 33.3 & 40.0 & \second{36.8} & \second{53.9} & 39.7 & \second{58.9} & 45.3 & 35.9 & \best{35.5} & \second{12.6} & 65.3 & 41.9 \\
\textbf{StepOPSD ($\lambda_{mix}=0.05$)} & 64.7 & 44.4 & 56.5 & \best{60.9} & 42.1 & \best{55.0} & \best{56.3} & 40.5 & 58.6 & \best{45.6} & 34.9 & 29.8 & 10.8 & 64.5 & 41.4 \\
\textbf{StepOPSD ($\lambda_{mix}=0.2$)} & 58.8 & 44.4 & 52.2 & 26.1 & \second{52.6} & 20.0 & 43.8 & \second{40.6} & \best{59.4} & 44.4 & \best{37.1} & 32.0 & 11.6 & 64.1 & \second{42.0} \\
\midrule
\multicolumn{16}{l}{\textit{Qwen2.5-3B-Instruct}} \\
Vanilla & 44.4 & 11.1 & 6.2 & 15.4 & 28.6 & 12.5 & 21.9 & 24.6 & 48.1 & 31.0 & 26.3 & 25.3 & 7.2 & 59.7 & 31.7 \\
Skill-Prompt* & 51.7 & 66.7 & 48.4 & 0.0 & 4.3 & 10.0 & 28.9 & 23.7 & 46.2 & 30.6 & 24.4 & 22.1 & 7.5 & 12.5 & 23.9 \\
OPSD & 48.8 & 41.7 & 16.7 & 0.0 & 15.8 & 16.7 & 23.1 & 0.1 & 0.1 & 0.1 & 0.0 & 0.0 & 0.0 & 0.0 & 0.0 \\
GRPO & 91.2 & 62.5 & \second{96.2} & 61.9 & 65.0 & 47.4 & 70.0 & 43.5 & 58.8 & 43.0 & 36.8 & 32.2 & 11.7 & \best{72.5} & 42.6 \\
Skill-GRPO & \second{98.9} & 71.4 & 58.8 & 70.6 & 40.7 & 29.2 & 60.2 & 44.3 & 59.6 & 44.3 & 39.0 & 36.1 & 14.5 & 66.4 & 43.5 \\
Skill-GRPO* & 94.3 & 57.1 & \best{100.0} & 66.7 & 73.1 & 57.1 & 80.5 & 44.3 & 59.6 & 44.3 & 39.0 & 36.1 & 14.5 & 14.9 & 36.1 \\
GRPO+OPSD & \best{100.0} & \best{82.4} & 85.7 & \second{75.0} & 70.0 & 60.0 & 81.2 & \second{44.9} & 61.2 & 45.2 & \best{40.4} & 38.5 & \second{16.0} & 66.1 & \second{44.6} \\
Skill-SD & 88.2 & 50.0 & \second{96.2} & 52.4 & 65.0 & 57.9 & 73.4 & 44.4 & 60.4 & 44.0 & \second{39.5} & \best{40.4} & 15.4 & 64.9 & 44.1 \\
RLSD & 87.9 & \second{75.0} & 90.9 & \second{75.0} & 73.1 & 68.4 & 79.7 & 41.5 & 58.6 & 42.3 & \best{40.4} & \second{40.2} & \best{16.8} & \second{66.9} & 43.8 \\
SDAR & 97.1 & 62.5 & \best{100.0} & 61.9 & \second{75.0} & 84.2 & \best{84.4} & 44.8 & 58.1 & 44.3 & 38.6 & 36.2 & 15.7 & 66.1 & 43.4 \\
\textbf{StepOPSD ($\lambda_{mix}=0.05$)} & 82.4 & 66.7 & 82.6 & 52.2 & 73.7 & 75.0 & 73.4 & 43.6 & 61.2 & 43.8 & 39.2 & 38.1 & 15.8 & 64.5 & 43.7 \\
\textbf{StepOPSD ($\lambda_{mix}=0.05$, $\alpha_{clip}=0.05$)} & -- & -- & -- & -- & -- & -- & -- & \best{45.0} & \best{61.6} & \second{46.2} & \second{39.5} & 39.5 & 14.4 & 65.3 & \best{45.7} \\
\textbf{StepOPSD ($\lambda_{mix}=0.2$)} & 45.0 & 55.6 & 87.0 & 65.2 & 57.9 & \second{85.3} & 69.5 & 44.4 & \second{61.4} & \best{46.8} & \best{40.4} & 38.1 & 14.1 & 65.0 & 44.3 \\
\textbf{StepOPSD ($\lambda_{mix}=0.2$, $\alpha_{clip}=0.05$)} & 97.1 & 66.7 & 87.0 & \best{79.1} & \best{78.9} & \best{95.0} & \second{83.6} & -- & -- & -- & -- & -- & -- & -- & -- \\
\bottomrule
\end{tabular}
}
\caption{Performance on ALFWorld and Search-QA across the 1.7B and 3B model scales. We report success rate (\%) on ALFWorld and accuracy (\%) on Search-QA. We additionally include two informative 3B StepOPSD variants with altered $\alpha_{clip}$; \texttt{--} indicates the corresponding task was not run for that follow-up.}
\label{tab:main-result}
\end{table*}

\subsection{Benchmarks}
To comprehensively evaluate our method, we adopt two diverse agentic environments: ALFWorld \citep{ALFWorld20} and Search-QA \citep{jin2025search}.
\textbf{ALFWorld} is a text-based embodied AI benchmark featuring 3,827 household tasks spanning six categories: Pick and Place (Pick), Look at Obj in Light (Look), Pick Clean then Place in Recep (Clean), Pick Heat then Place in Recep (Heat), Pick Cool then Place in Recep (Cool), and Pick Two Obj and Place (Pick2).
\textbf{Search-QA} aggregates multiple search-augmented question-answering datasets. It includes single-hop datasets—NQ \citep{kwiatkowski-etal-2019-natural}, TriviaQA \citep{joshi-etal-2017-triviaqa}, and PopQA \citep{mallen-etal-2023-trust}—as well as multi-hop reasoning tasks—HotpotQA \citep{yang-etal-2018-hotpotqa}, 2Wiki \citep{ho-etal-2020-constructing}, MuSiQue \citep{trivedi-etal-2022-musique}, and Bamboogle \citep{press-etal-2023-measuring}.

\subsection{Implementation Details}
We optimize \texttt{Qwen3-1.7B} and \texttt{Qwen2.5-3B-Instruct} with \method{} for up to 150 steps on 4 A800 GPUs. For ALFWorld, we use standard splits, sampling 16 tasks per batch with 8 rollouts per prompt, a maximum prompt length of 2,048 tokens, and a response cap of 512 tokens. For Search-QA, following \searchr{} \citep{jin2025search}, we use E5 as the dense retriever. The training split includes NQ and HotpotQA, while the remaining datasets are reserved for out-of-domain evaluation. Each batch contains 128 tasks with a maximum prompt length of 4,096 tokens.

Our implementation uses several configurations to stabilize credit shaping. We use a stale reference policy (\texttt{stale\_ref\_policy}, refreshed every 10 steps) as the teacher to avoid moving-target instability. Distillation targets isolate causal boundaries (\texttt{action\_only} for ALFWorld and \texttt{clean\_step\_no\_observation} for Search-QA) rather than full trajectories, and hindsight contexts are enriched with the first correct peer trajectory from the same rollout group. We modulate the RL advantage with a sigmoid weight function, using a default instantiation with $\alpha_{clip}=0.2$ and an initial $\lambda_{mix}=0.2$ that linearly decays over 50 steps; later ablations show that tighter $\alpha_{clip}$ is often more stabilizing, whereas the best $\lambda_{mix}$ remains environment-dependent. After the decay reaches zero, the stale teacher is still refreshed and StepOPSD statistics are still logged, but they no longer affect the actor update. We therefore read post-step-50 StepOPSD curves as diagnostics of teacher-student drift under a continuing stale teacher, rather than evidence that shaping remains equally active late in training. We also apply step-wise advantage normalization (\texttt{equal\_step\_mean\_abs}) to prevent long steps from dominating the optimization landscape. We set the GRPO KL penalty to $0.01$ and the invalid action penalty coefficient to $0.1$.

\subsection{Baselines}
To ensure a rigorously fair comparison under identical infrastructure, we benchmark \method{} against three groups of approaches across our base models:
\textbf{(1) Training-free methods:} \textit{Skill-Prompt} retrieves task-relevant skills via keyword matching (KM) and prepends them to the input prompt at inference time.
\textbf{(2) Post-training methods:} This includes standard \grpo{} \citep{schulman2017ppo}, vanilla OPSD, and \textit{Skill-GRPO}. \textit{Skill-GRPO} augments GRPO by retrieving skills and injecting them into the training prompt; at test time it can run with (\textit{Skill-GRPO*}) or without retrieved skills.
\textbf{(3) Hybrid methods:} These fuse RL with privileged knowledge distillation. We compare against naive \textit{GRPO+OPSD}, \textit{Skill-SD} \citep{wang2026skillsdskillconditionedselfdistillationmultiturn}, \textit{RLSD} \citep{yang2026rlsd}, and the current state-of-the-art \textit{SDAR} \citep{lu2026sdar}. \textit{GRPO+OPSD} simply adds the OPSD distillation loss as an auxiliary objective on top of GRPO training.

\subsection{Main Results}
Table~\ref{tab:main-result} reports the main comparison on ALFWorld and Search-QA across \texttt{Qwen3-1.7B} and \texttt{Qwen2.5-3B-Instruct}. The central pattern is not merely that embodied tasks prefer one setting while retrieval tasks prefer another. Rather, the benefit of \method{} varies systematically with the \emph{specific failure mode} of each subset: the larger the mismatch between a sparse trajectory-level reward and the true local decision that determines success, the larger the value of step-aware credit shaping.

\textbf{Fine-grained patterns on ALFWorld.}
Among the six embodied subsets, \texttt{Pick} and \texttt{Look} are relatively short-horizon and less credit-ambiguous: once the agent correctly grounds the target object or light source, the remaining path is direct. Standard RL or skill-augmented baselines therefore remain competitive there, e.g., SDAR reaches $73.5\%$ on \texttt{Pick} at 1.7B and GRPO+OPSD reaches $82.4\%$ on \texttt{Look} at 3B. By contrast, \texttt{Heat}, \texttt{Cool}, and \texttt{Clean} involve latent state transitions that are invisible in the final reward decomposition: an agent may correctly find, carry, and place the object, yet still fail because one intermediate action did not actually change its state. Here, trajectory-level RL is especially wasteful because it punishes the whole rollout for what is often a single local mistake. Step-aware shaping is most effective in exactly this regime. The clearest case is \texttt{Heat} at 1.7B, where \method{} with $\lambda_{mix}=0.05$ reaches $60.9\%$, well above SDAR ($33.3\%$) and GRPO ($40.0\%$), indicating that the post-hoc teacher signal corrects the decisive appliance-interaction step rather than globally suppressing exploration. At 3B, the reduced-$\alpha_{clip}$ run ($\lambda_{mix}=0.2$, $\alpha_{clip}=0.05$) strengthens this picture, reaching the best \texttt{Heat} score ($79.1\%$) and the best \texttt{Cool} score ($78.9\%$). \texttt{Clean} shows a milder version of the same pattern, suggesting that its key subgoal is easier for competitors to verbalize once discovered. \texttt{Pick2} exposes a different bottleneck: not hidden state transition, but cross-subgoal bookkeeping. After partially solving object 1, the agent must preserve progress while initiating and completing object 2. On 1.7B, the gentler $\lambda_{mix}=0.05$ achieves the best score ($55.0\%$), showing that weaker shaping better protects fragile exploration. On 3B, the reduced-$\alpha_{clip}$ run reaches the best \texttt{Pick2} score ($95.0\%$), suggesting that once the base policy is more competent, strong global guidance remains effective if local over-correction is tightly bounded. Together, these patterns explain why ALFWorld benefits strongly from tighter local clipping at 3B, even though the low-$\lambda_{mix}$ rows still reveal what happens when that control is absent.

\textbf{Fine-grained patterns on Search-QA.}
The Search-QA results show an equally sharp but different decomposition. \texttt{NQ} is relatively short-chain and entity-centric, so the main bottleneck is often whether the agent retrieves the answer once, not whether it sustains a long tool-use chain. StepOPSD is therefore competitive but not always dominant there under its default setting. The reduced-$\alpha_{clip}$ 3B follow-up sharpens this picture: with $\lambda_{mix}=0.05$ and $\alpha_{clip}=0.05$, \method{} reaches the best \texttt{NQ} score ($45.0\%$), suggesting that this subset benefits from preserving search freedom once local over-correction is suppressed. \texttt{TriviaQA}, by contrast, is highly sensitive to lexical formulation, aliases, and evidence phrasing; a slightly better search query often flips the entire example. This is exactly where step-level shaping should matter most, and \method{} indeed reaches the best \texttt{TriviaQA} scores at both scales, including $61.6\%$ for the reduced-$\alpha_{clip}$ 3B run. \texttt{PopQA} follows a similar logic: many questions are factual and seemingly shallow, but they are long-tail and brittle to query wording, so better credit assignment at the precise retrieval step again helps \method{} reach the best 1.7B score ($45.6\%$ with $\lambda_{mix}=0.05$) and remain near-best at 3B, with $46.8\%$ for $\lambda_{mix}=0.2$ and $46.2\%$ for the reduced-$\alpha_{clip}$ $\lambda_{mix}=0.05$ follow-up. \texttt{HotpotQA} makes the mechanism even clearer. Because the second hop depends on an intermediate bridge entity, one malformed first query can poison the rest of the rollout. A denser teacher signal is therefore disproportionately valuable, enabling \method{} with $\lambda_{mix}=0.2$ to achieve the best 1.7B score ($37.1\%$) and tie the best 3B score ($40.4\%$). By contrast, \texttt{2Wiki} and especially \texttt{MuSiQue} remain harder for our current variant. Their failures are not always caused by one bad search action; they often arise from multi-hop evidence composition under stronger distractors and deeper dependency chains. In those cases, local step correction still helps, but its advantage is diluted when the dominant error lies in global evidence aggregation rather than in a single retrieval decision. \texttt{Bamboogle} is also revealing: once the model issues one sufficiently targeted search, strong RL baselines already perform well, so the room for improvement shrinks, and even our reduced-$\alpha_{clip}$ 3B follow-up remains below GRPO's $72.5\%$.

\textbf{A consistent task-dependent law across scales.}
Taken together, Table~\ref{tab:main-result} reveals a sharper empirical law than a simple weak-vs.-strong shaping dichotomy. Embodied subsets with rigid physical dynamics consistently benefit from tighter local control: at 3B, keeping $\lambda_{mix}=0.2$ but tightening $\alpha_{clip}$ to $0.05$ raises the finalized ALFWorld average to $83.6\%$ and produces the best \texttt{Heat}, \texttt{Cool}, and \texttt{Pick2} scores. Retrieval subsets remain more sensitive to the interaction between global mixing and local clipping. Under the default $\alpha_{clip}$, $\lambda_{mix}=0.2$ yields the stronger 3B Search-QA average ($44.3\%$ vs.\ $43.7\%$), but once $\alpha_{clip}$ is tightened to $0.05$, the lighter-$\lambda_{mix}$ variant rises to the best 3B Search average ($45.7\%$). The transferable pattern is therefore not that one global $\lambda_{mix}$ universally wins, but that tighter local clipping through smaller $\alpha_{clip}$ is broadly beneficial, while the optimal global $\lambda_{mix}$ remains task-dependent. This is the central empirical message of \method{}: step-aware distillation is most effective when trajectory-level rewards are maximally misaligned with the local causal action that determines downstream success.

\subsection{Training Diagnostics: Pre- and Post-Step 50 Phase Transition}
\begin{figure}[t]
    \centering
    \includegraphics[width=\linewidth]{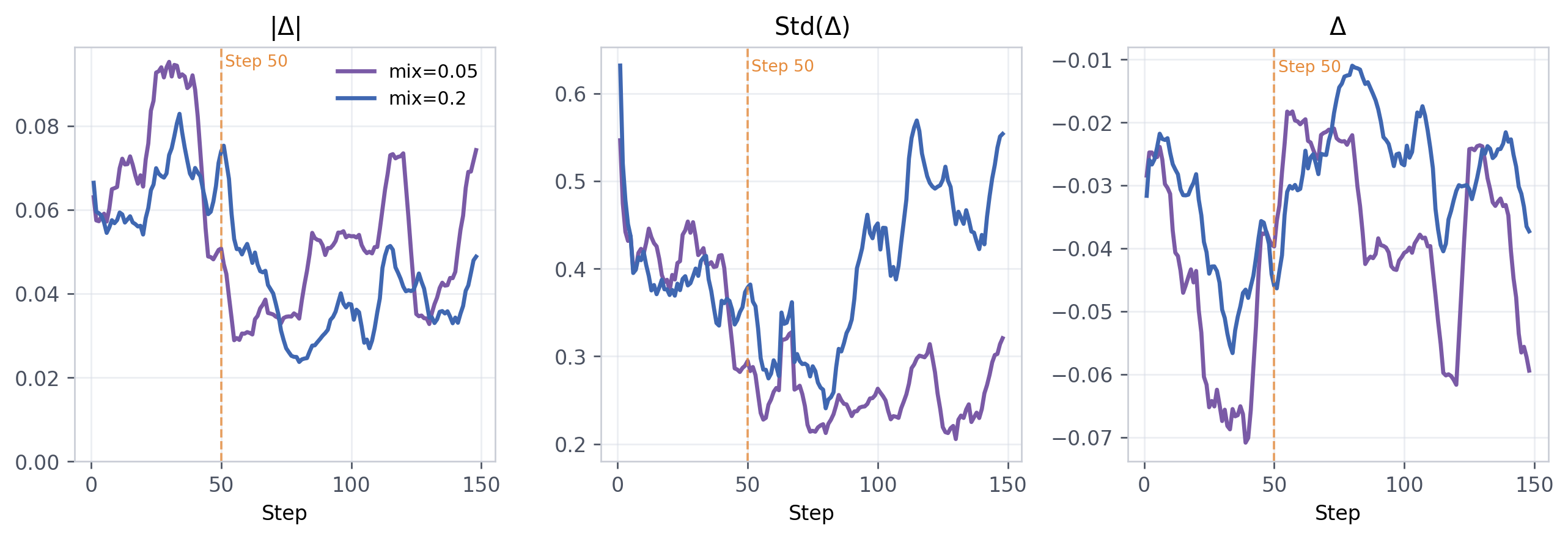}
    \caption{StepOPSD dynamics on ALFWorld for \texttt{Qwen3-1.7B}. Around step 50, heavy shaping ($\lambda_{mix}=0.2$) induces a variance explosion in the teacher-student gap.}
    \label{fig:alfworld-dynamics}
\end{figure}

End-task accuracy alone does not reveal how \method{} works. We therefore track the standard deviation of the teacher-student gap, $\text{Std}(\Delta)$, over training. The key pattern is a phase transition around step 50, when the base RL policy begins to mature. This turning point must be read with the implementation detail above: by step 50, the explicit StepOPSD mixing coefficient has decayed to zero, so later curves mainly reflect student drift relative to a stale teacher still refreshed every 10 steps, rather than a still-active shaping term.

In Search-QA, sparse rewards leave the student more dependent on the teacher. Before step 50, all models show similar variance ($\text{Std}(\Delta)\approx 0.40$--$0.49$). After step 50, the weakly shaped student ($\lambda_{mix}=0.05$) jumps to $0.61$, while stronger shaping ($\lambda_{mix}=0.2$) remains stable at about $0.44$. Without enough guidance, exploration drifts.

In ALFWorld, the pattern reverses. Once the student learns basic controls, heavy shaping ($\lambda_{mix}=0.2$) creates a ``tug-of-war'' between hindsight paths and RL exploration, pushing $\text{Std}(\Delta)$ from $0.49$ to $0.61$ after step 50. By contrast, milder shaping, via lower $\lambda_{mix}$ or tighter clipping, lets the policy optimize reward without over-constraint and drives $\text{Std}(\Delta)$ down (e.g., from $0.46$ to $0.26$).

This mirrored phase transition across environments supports our theoretical variance bound (Theorem~\ref{thm:variance}): \method{} reduces gradient variance only when distillation strength is calibrated to the density of the environment reward.

\subsection{Ablation Studies: Weight Clipping as a Cross-Task Stabilizer}
\begin{figure}[t]
    \centering
    \includegraphics[width=\linewidth]{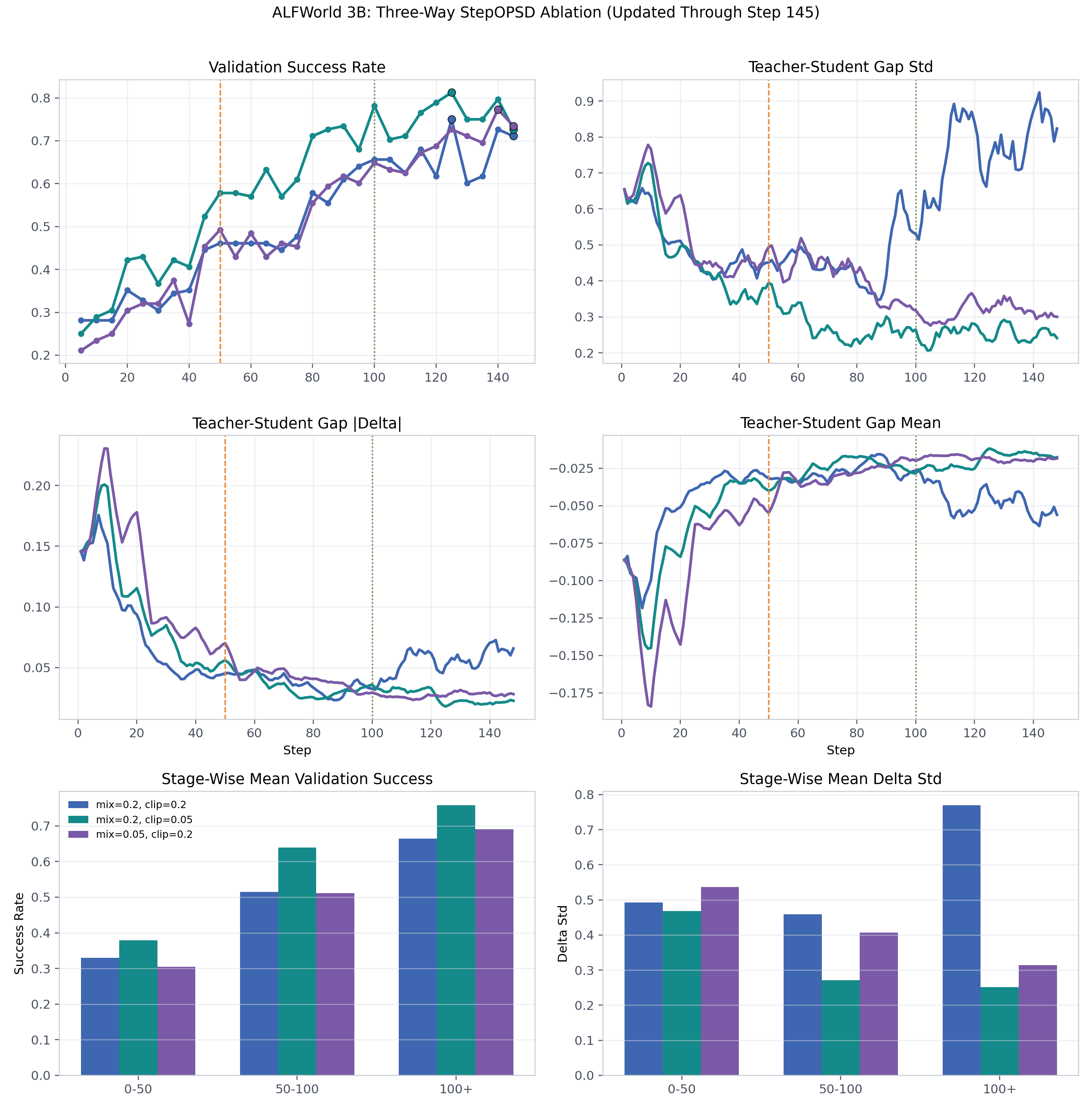}
    \caption{Three-way StepOPSD ablation on ALFWorld with \texttt{Qwen2.5-3B-Instruct}. We compare the heavy baseline ($\lambda_{mix}=0.2$, $\alpha_{clip}=0.2$), the lower-$\lambda_{mix}$ variant ($\lambda_{mix}=0.05$, $\alpha_{clip}=0.2$), and the reduced-$\alpha_{clip}$ variant ($\lambda_{mix}=0.2$, $\alpha_{clip}=0.05$). The validation panel marks peak and final checkpoints; the others show training dynamics and stage-wise summaries over steps 0--50, 50--100, and 100+.}
    \label{fig:alfworld-ablation}
\end{figure}

We use ALFWorld at 3B because it gives the cleanest three-way comparison among StepOPSD control knobs: a heavy baseline $(\lambda_{mix}=0.2, \alpha_{clip}=0.2)$, a lower-$\lambda_{mix}$ variant $(\lambda_{mix}=0.05, \alpha_{clip}=0.2)$, and a reduced-$\alpha_{clip}$ variant $(\lambda_{mix}=0.2, \alpha_{clip}=0.05)$. Figure~\ref{fig:alfworld-ablation} shows that they stay close early but separate clearly after step 50. The reduced-$\alpha_{clip}$ run yields the best mature-phase validation and the highest logged peak.

The dynamics show that the main issue is not shaping itself, but unbounded local correction during the active phase. After step 50, the heavy baseline reaches a late-stage gap variance of $\text{Std}(\Delta)=0.770$ over steps 101--150, whereas the reduced-$\alpha_{clip}$ run stays at $0.252$. Because $\lambda_{mix}$ has already decayed to zero, we read these later curves as drift against stale-teacher snapshots rather than continued active shaping. The reduced-$\alpha_{clip}$ run matters because it leaves behind a more stable mature policy.

\begin{figure}[t]
    \centering
    \includegraphics[width=\linewidth]{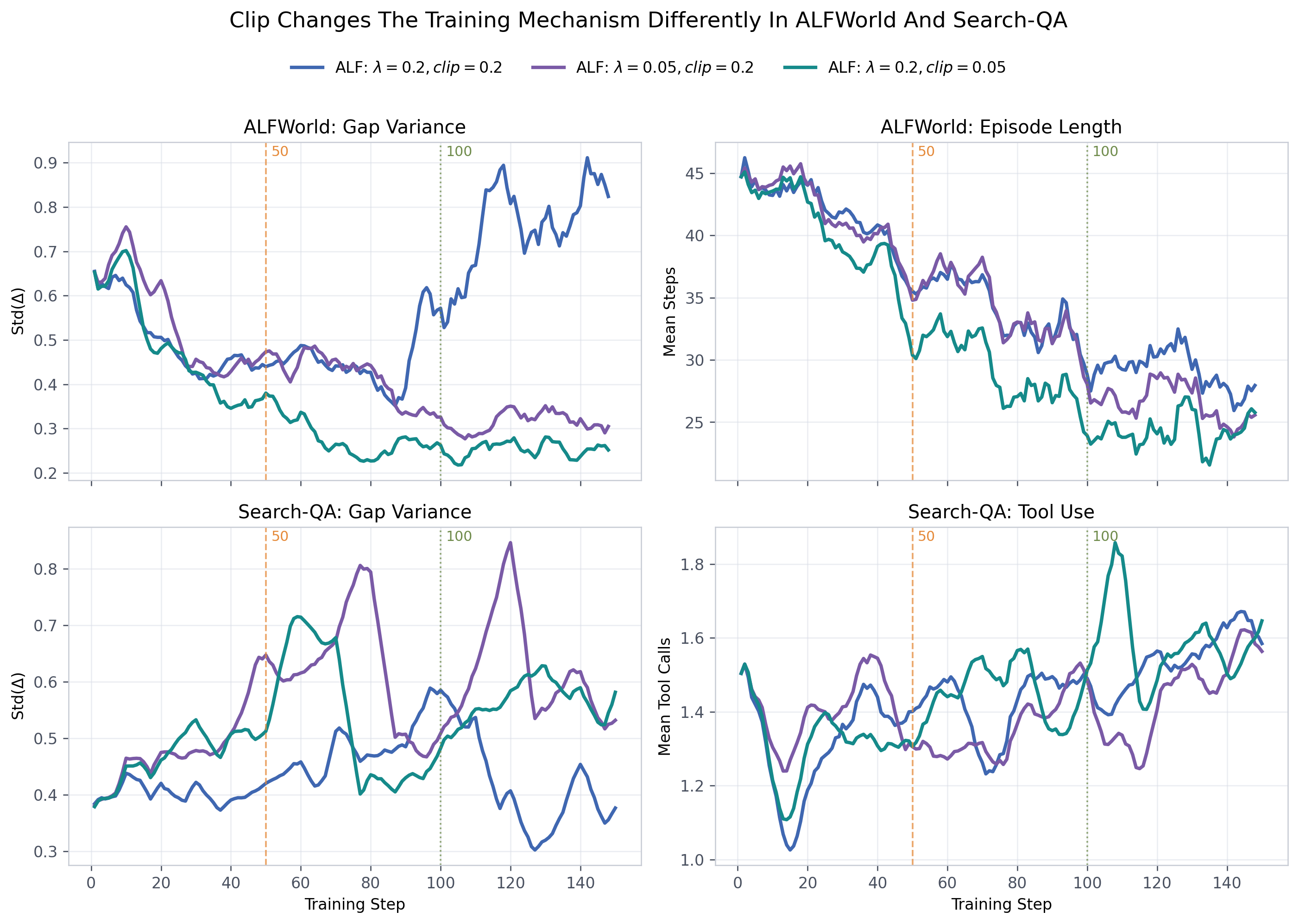}
    \caption{Mechanistic view of weight clipping at 3B. In ALFWorld, tightening $\alpha_{clip}$ under fixed $\lambda_{mix}=0.2$ prevents the late-stage explosion of $\mathrm{Std}(\Delta)$ and shortens embodied trajectories. In Search-QA, tightening $\alpha_{clip}$ under fixed $\lambda_{mix}=0.05$ lowers gap variance relative to the loose-clip counterpart and increases tool use. The phase markers at steps 50 and 100 show that clipping matters most after the policy enters the mature regime.}
    \label{fig:clip-mechanism-3b}
\end{figure}

Figure~\ref{fig:clip-mechanism-3b} shows the same mechanism across tasks. In ALFWorld, clipping turns strong shaping into bounded strong shaping: the reduced-$\alpha_{clip}$ run keeps the lowest late-stage variance and shortens the mature policy from $29.2$ to $24.3$ actions, without changing the near-perfect valid-action ratio. In Search-QA, the effect is different but consistent. Under weak global mixing, tightening $\alpha_{clip}$ lowers late-stage gap variance, raises tool use from $1.45$ to $1.59$, and shortens responses from $98.9$ to $86.4$ tokens. The gain therefore comes less from formatting correctness than from reallocating behavior toward actual search.

Taken together, the 3B ablations reveal a more specific trade-off than ``more vs.\ less distillation.'' The global mixing coefficient $\lambda_{mix}$ controls how strongly the teacher anchors the update, whereas the clipping coefficient $\alpha_{clip}$ controls how far any local token correction can deviate from the RL signal. Tight clipping helps in both environments because it prevents a small subset of tokens from dominating the post-rollout correction; what changes across tasks is whether this bounded correction stabilizes embodied action plans or encourages more tool-centric retrieval behavior.

\FloatBarrier

\section{Related Work}
\subsection{Agentic RL and Credit Assignment}
Agentic language models act through temporally extended trajectories rather than single-pass responses. In formulations such as ReAct, these trajectories are naturally structured as interleaved reasoning, actions, and observations \citep{yao2023react,schick2023toolformer}. Recent RL work has scaled this paradigm to search agents and other long-horizon settings \citep{jin2025search,jin2025empirical}, but also makes the credit-assignment problem more severe: rewards are sparse and trajectory-level, while success often depends on a few local decisions. Our work targets this mismatch directly.

\subsection{On-Policy Distillation and Self-Distillation}
Policy distillation transfers behavior through distribution matching \citep{rusu2015policy}, while more recent on-policy variants apply supervision on student-generated rollouts. In particular, OPSD-style methods use privileged or hindsight contexts to define a stronger teacher on the student's own trajectories \citep{zhao2026opsd,pidistill2026}. For multi-turn agents, Skill-SD, RLSD, and SDAR all study how dense self-distillation can complement RL \citep{wang2026skillsdskillconditionedselfdistillationmultiturn,yang2026rlsd,lu2026sdar}. StepOPSD belongs to this line, but localizes supervision to extracted action-centered step spans instead of distilling an entire trajectory.

\subsection{Reward Shaping and Advantage Reweighting}
Our method is also related to reward shaping and policy-preserving auxiliary guidance. Classical work showed that shaping signals can accelerate learning without changing the original optimum when designed appropriately \citep{ng1999rewardshaping,wiewiora2003potential}. Modern policy optimization likewise uses comparative signals to modulate updates without replacing the base objective \citep{schulman2017ppo,rafailov2024dpo,azar2024ipo}. StepOPSD follows this principle: RL remains the backbone, while teacher-derived token gaps modulate post-rollout advantages.

\subsection{Step-Structured Supervision for Agents}
Multi-turn agent trajectories are heterogeneous: observations, tool calls, reasoning, and answers do not play the same causal role. Supervising the full sequence uniformly can therefore blur the boundary between context and action. Recent agent distillation methods introduce richer training-time contexts, such as privileged information, skill summaries, or token gating \citep{pidistill2026,wang2026skillsdskillconditionedselfdistillationmultiturn,lu2026sdar}. Our contribution is to make the supervision unit explicit by extracting action-centered step segments and injecting their teacher-student comparison directly into advantage shaping.

\section{Conclusion}
This paper presents \method{}, a step-aware extension of Online Preference Distillation for multi-turn agent reinforcement learning against Online Policy Distilation (OPD). Instead of distilling an entire trajectory uniformly, \method{} isolates the action-centered spans that causally determine agent behavior and feeds hindsight teacher signals back into token-level advantage shaping. The method is faithful to the underlying \grpo{} training stack, modular in implementation, and naturally aligned with the heterogeneous structure of agent trajectories.

Empirically, our results show that the main value of StepOPSD is not a uniform gain on every benchmark average, but a more precise correction of the subsets most sensitive to local decision errors. Across ALFWorld and Search-QA, \method{} consistently achieves multiple best and second-best results, while revealing a sharper two-knob law: tighter local clipping through smaller $\alpha_{clip}$ is broadly beneficial, whereas the optimal global $\lambda_{mix}$ remains task-dependent. This supports our central claim that step-aware distillation is most effective when sparse trajectory-level rewards are weakly aligned with the local causal action that determines downstream success.

\section*{Limitations}
While \method{} demonstrates strong potential for step-aware credit assignment, we acknowledge two main limitations in the current study.

First, our ablation studies indicate that the optimal mixing strength ($\lambda_{mix}$) varies across different environment types---requiring gentler shaping for action-centric tasks and stronger shaping for knowledge-intensive retrieval tasks. This introduces an additional hyperparameter tuning step when deploying the method to entirely new domains.

Second, the current framework focuses strictly on post-rollout advantage shaping. While this ensures stable integration with the underlying RL backbone, it does not dynamically intervene during the generation phase. Exploring how to use the step-level distillation signal to guide real-time decoding remains an important direction for future work.

\bibliography{custom}

\appendix

\section{Analytical Discussion}
\label{sec:theory}
In this appendix, we provide a formal analysis of \method{}'s properties, specifically focusing on how it modulates the reinforcement learning gradients without shifting the optimal policy, and how it mitigates the credit assignment problem in long-horizon agent trajectories.

Let the standard policy gradient for a token $y_t$ be $g_{t} = A_t \nabla_\theta \log \pi_\theta(y_t | s_t)$. \method{} modifies this advantage to $\tilde{A}_t = \Psi_t A_t$, where $\Psi_t = 1 - \lambda_{mix} + \lambda_{mix} w_t$ is the shaping multiplier, $w_t = 2\sigma(\text{sign}(A_t) \cdot \Delta_t) \in (0, 2)$ is the weight function, and $\Delta_t = \log \pi_{ref}(y_t | s^+_t) - \log \pi_\theta(y_t | s_t)$ is the teacher-student gap.

\subsection{Property 1: Sign Preservation and Directional Consistency}
A critical requirement for any advantage shaping method is that it must not introduce pathological local optima that deviate from the true environment reward.

\begin{proposition}[Sign Preservation]
\label{prop:sign}
For any $\lambda_{mix} \in [0, 1)$, the shaped advantage $\tilde{A}_t$ strictly preserves the sign of the original RL advantage $A_t$, i.e., $\text{sign}(\tilde{A}_t) = \text{sign}(A_t)$.
\end{proposition}
\begin{proof}
By definition, the shaping multiplier is $\Psi_t = 1 - \lambda_{mix} + \lambda_{mix} w_t$. Since $w_t > 0$ and $\lambda_{mix} \in [0, 1)$, we are guaranteed that $1 - \lambda_{mix} > 0$, and thus $\Psi_t > 0$. Consequently, multiplying $A_t$ by a strictly positive scalar $\Psi_t$ preserves its sign: $\tilde{A}_t = 0 \iff A_t = 0$, and $\tilde{A}_t > 0 \iff A_t > 0$.
\end{proof}

\begin{theorem}[Directional Consistency]
\label{thm:consistency}
Let $g_{RL} = \mathbb{E}[A_t \nabla_\theta \log \pi_\theta]$ and $g_{StepOPSD} = \mathbb{E}[\tilde{A}_t \nabla_\theta \log \pi_\theta]$. Under the assumption that the shaping multiplier $\Psi_t$ is bounded and positive, the expected StepOPSD gradient maintains a non-negative cosine similarity with the exact RL gradient, acting as a valid descent direction for the original RL objective.
\end{theorem}
\begin{proof}
Because $\tilde{A}_t = \Psi_t A_t$ and $\Psi_t > 0$ (Proposition~\ref{prop:sign}), the gradient step for every single token is scaled by a strictly positive magnitude. In expectation, reweighting the components of a gradient vector by strictly positive scalars ensures that the modified vector lies in the same half-space as the original vector ($\langle g_{RL}, g_{StepOPSD} \rangle \ge 0$). Thus, \method{} optimizes toward the same global reward maximum as standard \grpo{}, merely altering the traversal path across the loss landscape.
\end{proof}

\subsection{Property 2: Variance Reduction in Multi-Turn Credit Assignment}
In multi-turn tasks, the standard advantage $A_t$ for a step $k$ is estimated using the trajectory return, introducing high variance due to the compounding uncertainty of future steps $k+1 \dots K$.

\begin{theorem}[Variance Bound via Hindsight Teacher]
\label{thm:variance}
Let $A_t = A^*_t + \epsilon_t$ where $A^*_t$ is the oracle local advantage of the token and $\epsilon_t \sim \mathcal{N}(0, \sigma^2)$ is the noise from future rollout variability. If the teacher's gap $\Delta_t$ provides an unbiased signal of $A^*_t$ such that the weight $w_t$ dampens (i.e., $\Psi_t < 1$) when $\text{sign}(A_t) \neq \text{sign}(A^*_t)$, the variance of the StepOPSD gradient estimator is bounded below the standard estimator: $\text{Var}(g_{StepOPSD}) < \text{Var}(g_{RL})$.
\end{theorem}
\begin{proof}
The variance of the standard gradient is proportional to $\mathbb{E}[A_t^2] = (A^*_t)^2 + \sigma^2$. In \method{}, the multiplier $\Psi_t$ suppresses the advantage magnitude precisely when the trajectory-level return contradicts the local step-level teacher evaluation ($\text{sign}(A_t) \cdot \Delta_t < 0$, yielding $w_t < 1$ and thus $\Psi_t < 1$). By actively discounting these high-noise, low-confidence token updates, the contribution of the noise term $\epsilon_t$ to the gradient variance is scaled down by $\mathbb{E}[\Psi_t^2 | \text{contradiction}] < 1$, yielding a lower overall gradient variance. 
\end{proof}
This mathematical property formally explains why the stronger mixing strength ($\lambda_{mix}=0.2$) achieves superior performance in complex Search-QA tasks: it effectively mitigates the credit-assignment noise $\sigma^2$ inherent to long, multi-document reasoning chains.

\end{document}